\documentclass[sigconf]{acmart}

\usepackage{booktabs} 

\usepackage{bm}
\usepackage{stackengine} 
\usepackage[algo2e,ruled]{algorithm2e}

\newcommand{\tss}{\mathrm{TSS}}
\newcommand{\mymin}[1]{\min_{\scalebox{1.0}{$#1$}}}
\newcommand{\myargmin}[1]{\min_{\scalebox{1.0}{$#1$}}}

\newcommand{\Graph}{\mathcal{G}} 
\newcommand{\Y}{\mathcal{Y}} 
\newcommand{\D}{\mathcal{D}} 
\newcommand{\I}{\mathbb{I}} 
\renewcommand{\S}{\mathcal{S}} 
\renewcommand{\H}{\mathcal{H}} 

\makeatletter
\def\thickhline{%
  \noalign{\ifnum0=`}\fi\hrule \@height \thickarrayrulewidth \futurelet
   \reserved@a\@xthickhline}
\def\@xthickhline{\ifx\reserved@a\thickhline
               \vskip\doublerulesep
               \vskip-\thickarrayrulewidth
             \fi
      \ifnum0=`{\fi}}
\makeatother

\newlength{\thickarrayrulewidth}
\renewcommand{\v}[1]{\textcircled{\raisebox{-0.7pt}{\footnotesize{#1}}}}
\newcommand{\e}{\text{--}}

\newcommand{\gxxx}[3]{\v{#1}\e\v{#2}\e\v{#3}}

\renewcommand{\ss}[2]{\addstackgap{\shortstack{#1\\#2}}}

\newcommand{\fst}[1]{{\bf #1}}
\newcommand{\fss}[2]{\addstackgap{\shortstack{\bf #1\\#2}}}

\SetKwData{t}{tss}
\SetKwData{mt}{min\_tss}
\SetKwData{bg}{best\_$g$}
\SetKwData{g}{$g$}
\SetKwData{n}{n}
\SetKwData{bound}{bound}

\newcommand{\TODO}[1]{\textcolor{green!50!black}{{\footnotesize TODO }#1}}
\newcommand{\COMMENT}[1]{\textcolor{orange!80!black}{{\footnotesize Comment: }#1}}
\renewcommand{\TODO}[1]{} \renewcommand{\COMMENT}[1]{} 
\newcommand{\SKIP}[1]{#1}
\renewcommand{\SKIP}[1]{} 

\setcopyright{none}

\acmDOI{}

\acmISBN{}

\acmConference[MLG'18]{International Workshop on Mining and Learnig with Graphs}{August 2018}{London, United Kingdom}
\acmYear{}

\acmArticle{}
\acmPrice{}

\editor{}

\begin{document}
\title{Jointly learning relevant subgraph patterns and nonlinear models of their indicators}
\titlenote{The present study was supported in part by JSPS/MEXT KAKENHI \#17H01783, \#17K19953, and JST PRESTO \#JPMJPR15N9.}

\author{Ryo Shirakawa}
\affiliation{%
  \institution{Graduate School of Information Science and Technology, Hokkaido University}
}
\email{sira@art.ist.hokudai.ac.jp}

\author{Yusei Yokoyama}
\affiliation{%
  \institution{Research \& Development Group, Hitachi, Ltd}
  \institution{Graduate School of Information Science and Technology, Hokkaido University}
}
\email{yusei.yokoyama.qk@hitachi.com}

\author{Fumiya Okazaki}
\affiliation{%
  \institution{Graduate School of Information Science and Technology, Hokkaido University}
}
\email{fokazaki.w.h.i@gmail.com}

\author{Ichigaku Takigawa}
\affiliation{%
  \institution{Graduate School of Information Science and Technology, Hokkaido University}
  \institution{PRESTO, Japan Science and Technology Agency (JST)}
}
\email{takigawa@ist.hokudai.ac.jp}

\begin{abstract} 
  Classification and regression in which the inputs are graphs of arbitrary size and shape have been paid attention in various fields such as computational chemistry and bioinformatics. Subgraph indicators are often used as the most fundamental features, but the number of possible subgraph patterns are intractably large due to the combinatorial explosion. We propose a novel efficient algorithm to jointly learn relevant subgraph patterns and nonlinear models of their indicators. Previous methods for such joint learning of subgraph features and models are based on search for single best subgraph features with specific pruning and boosting procedures of adding their indicators one by one, which result in linear models of subgraph indicators. In contrast, the proposed approach is based on directly learning regression trees for graph inputs using a newly derived bound of the total sum of squares for data partitions by a given subgraph feature, and thus can learn nonlinear models through standard gradient boosting. An illustrative example we call the Graph-XOR problem to consider nonlinearity, numerical experiments with real datasets, and scalability comparisons to na\"ive approaches using explicit pattern enumeration are also presented.
\end{abstract}


\settopmatter{printacmref=false}
\maketitle

\section{Introduction}

Graphs are fundamental data structures for representing combinatorial
objects. However, precisely because of their combinatorial nature, it is usually
difficult to understand the underlying trends in large datasets of
graphs. The rapid increase in data in recent years also includes data
represented as graphs, and thus supervised learning in which the inputs are graphs of
arbitrary size and shape has gained considerable attention.
This problem commonly arises in diverse fields such as cheminformatics
\cite{Kashima:2003, Tsuda:2007,Saigo:2008a, Saigo:2009,Mahe:2009, Vishwanathan:2010, Shrvashidze:2011, Takigawa:2013, takigawa:2017}, and bioinformatics
\cite{Borgwardt:2005, Karklin:2005, Takigawa:2011b} as well as wide computer-science applications such as computer vision
\cite{Harchaoui:2007, Nowozin:2007, Barra:2013, Bai:2014a}, and natural language processing
\cite{Kudo:2005}. 

The present paper investigates the supervised learning of a function $f: \Graph \to \Y$
from finite pairs of input graphs and output values,
where $\Graph$ is a set of graphs
and $\Y$ is a label space such as $\{-1,+1\}$ and $\mathbb{R}$. In general settings, 
the most fundamental and widely used features are indicators of subgraph patterns.
Since the number of possible subgraph patterns are intractably large due to the combinatorial explosion, 
we need to use a heuristically limited class of subgraph patterns or to search for relevant patterns during the learning phase.

In addition to extensive studies on \textit{graph kernels} 
\cite{Kashima:2003, Borgwardt:2005, Harchaoui:2007,Mahe:2009, Vishwanathan:2010, Shrvashidze:2011,Barra:2013,Bai:2014a},
joint learning of relevant subgraph patterns and classification/regression models by their indicators
has also been developed \cite{Kudo:2005, Nowozin:2007,Saigo:2008a,Saigo:2009,takigawa:2017}.
This approach would not overlook any important features, but need some technical tricks to efficiently
search for relevant subgraph patterns from combinatorially huge candidates.
The previous methods use $\ell_1$ regularization for \textit{linear} models of all possible subgraph indicators, and thus can select relevant subgraph patterns. In contrast, any practical graph kernels are based on all subgraphs in a predefined class,
and do not try to select some relevant subsets of subgraph features. Note that the \textit{all-subgraphs} 
kernel is known to be theoretically hard\cite{Gartner:2003}.

In the present paper, 
we investigate \textit{nonlinear} models with all possible
subgraph indicators. The following are the contributions of the present study:
\begin{itemize}
\item
  We present two lesser-recognized facts to make sure the
  difference between linear and nonlinear models of substructural indicators:
  (1) For a closely related problem of supervised learning from
  itemsets, the hypothesis space of the nonlinear model of
  all possible sub-itemset indicators is equivalent to that
  of the linear model; (2) Nevertheless, for the indicators
  of connected subgraphs, the hypothesis space of the
  nonlinear model is strictly larger than that of the linear
  model. {\bf (Section \ref{sec:linearNonlinear})}
\item
  We develop a novel efficient supervised learning algorithm for joint learning of all relevant subgraph
  features and a nonlinear models of their indicators. Unlike existing approaches based on
  $\ell_1$-regularized linear models, the proposed algorithm
  is based on gradient tree boosting with base regression
  trees selecting each splitter out of all subgraph
  indicators with an efficient pruning based on the new bound
  in Theorem
  \ref{thm:bound}. {\bf (Section \ref{sec:proposed})}
\item
  We empirically demonstrate that
  (i) for the Graph-XOR dataset, the proposed nonlinear
  method actually outperforms several linear methods, which
  implies the existence of problems requiring nonlinear
  hypotheses, (ii) For several real datasets, we also
  observe similar superiority of the nonlinear models for some datasets, while it also turns out that
  the performance of linear models is fairly comparable for some datasets.
  {\bf (Section \ref{sec:experi})}
\end{itemize}

\subsection{Related Research and Our Motivation}
\label{sec:previousModel}

Although not discussed explicitly, most previous studies
\cite{Kudo:2005, Nowozin:2007,Saigo:2008a,Saigo:2009,takigawa:2017}
yielded linear models with respect to subgraph indicators as Boolean variables. However,
this would not be obvious at first glance because these studies were based on \textit{boosting} such as Adaboost
\cite{Kudo:2005} and LPBoost \cite{Saigo:2009}, which are usually expected
to produce nonlinear models. However, this is not the case because these studies used \textit{decision stumps} with respect to a single subgraph feature as base learners. 

The research of the present paper starts with our observation that replacing the decision
stumps in these existing methods with decision trees is far from straightforward. This is because the previous methods are based on efficient pruning with specifically derived bounds to find a single best subgraph pattern, and use the indicator as a base learner at each iteration.

One na\"ive method to obtain nonlinear models of subgraph indicators is to enumerate some candidate subgraphs from training graphs, explicitly construct 0-1 indicator-feature vectors of test graphs by solving subgraph-isomorphism directly, and apply a general nonlinear supervised learning to those feature vectors. The performance with all small-size subgraphs occurred in the given graphs is known to be comparable for cheminformatics datasets\cite{Wale:2008}. However these approaches would not scale well as we see later in Section \ref{sec:scalability}. Another good known heuristic idea is to use $r$-neighborhood subgraphs with radius $r$ at each node as seen in ECFP\cite{Rogers:2010} and graph convolutions\cite{Kearnes2016,gilmer:2017}. Unfortunately, the complete enumeration would not scale well either in this case, and usually requires some tricks such as feature hashing, feature folding, or feature embedding through neural nets, all of which are very interesting approaches but beyond the scope of this paper.

Note that it is not difficult to use the number of occurrences of subgraph $g$ in $G$ as
features instead of just 0-1 subgraph
indicators. Although not discussed herein, this case can be
investigated as a weighted version of indicators, and similar properties
would hold. 

\section{Preliminaries}

\subsection{Notations}

Let $[n]$ be $\{1,2,\dots,n\}$, and let $\I(P)$ denote the
indicator of $P$, i.e., $\I(P)=1$ if $P$ is
true, else $0$. We denote as $G \sqsupseteq g$ the
subgraph isomorphism that $G$ contains a subgraph that is isomorphic to
$g$ and its negation as $G \not\sqsupseteq g$. Thus, a subgraph
indicator $\I(G \sqsupseteq g) = 1$ if $G \sqsupseteq g$, otherwise 0.
We also denote the training set of input graphs $G_i \in \Graph$ and output responses $y_i \in \Y$ as
\begin{equation}
  \label{eq:train_data}
  \D = \{(G_1,y_1),(G_2,y_2),\dots,(G_N,y_N)\}, 
\end{equation}
where $\Graph$ is a set of all finite-size, connected, discretely-labeled,
undirected graphs. We denote $\Graph_N = \{G_i\mid i \in [N]\}$, 
and the set of all possible connected subgraphs as $\S_N = \bigcup_{G \in \Graph_N} \{g \mid G \sqsupseteq g\}$.


\subsection{Search Space for Subgraphs}
\label{sec:subgraphMining}

In supervised learning from graphs, we represent each input graph $G_i \in
\Graph_N$ by the characteristic vector $(\I(G_i \sqsupseteq g) \mid g \in
\S) $ with a set $\S$ of relevant subgraph features. However, since $\S$ is not
explicitly available when the learning phase starts, we need to
jointly search and construct $\S$ during the learning process.
In order to define an efficient search space for $\S_N$, i.e., any subgraphs occurring in $\Graph_n$,
the techniques for \textit{frequent subgraph mining}, which enumerates all
subgraphs that appear in more than $m$ input graphs for a given $m$, are
useful. Note that any subgraph feature
$g \in \S_N$ can occur multiple times at multiple locations in a single graph,
but $\I(G_i \sqsupseteq g) = 1$.

In the present paper, we use the search space of the gSpan algorithm \cite{Yan:2002},
which performs a depth-first search on the tree-shaped search spaces on $\S_N$,
referred to collectively as an \textit{enumeration tree}, as shown in Figure~\ref{fig:enum}.
Each node of the enumeration tree holds a subgraph feature $g'$ that
extends the subgraph feature $g$ at the parent node by one edge, namely, $ g' \sqsupseteq g $.
\begin{figure}[t]
  \centering
  \includegraphics[width=0.6\linewidth]{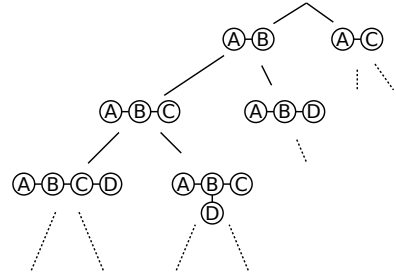}
  \caption{An enumeration tree}
  \label{fig:enum}
\end{figure}
The following \textit{anti-monotone} \TODO{anti-monotone?} property of
subgraph isomorphism over the enumeration tree on $\S_N$ can be used to
derive the efficient search-space pruning of the gSpan algorithm:
\begin{equation}
  \label{eq:propSubgraph}
  G_i \not\sqsupseteq g \Rightarrow G_i \not\sqsupseteq g'
  \quad \text{for} \quad
  g' \sqsupseteq g . 
\end{equation}

\subsection{Gradient Tree Boosting}
\label{sec:GTB}

Gradient tree boosting (GTB) \cite{friedman2001greedy,
  mason2000boosting} is a general algorithm for supervised learning to
predict a response $y$ from a predictor $\bm{x}$. 
For a given hypothesis space $\H$, the goal is to minimize the empirical
risk $L(f) = N^{-1}\sum_{i \in [N]} \ell(y_i,f(\bm{x}_i))$ of $f \in
\H$, which is the average of a loss function $\ell(y,f(\bm{x}))$ over
the training data $\{(\bm{x}_i, y_i)\}_{i \in[N]}$.

GTB is an additive ensemble model of regression trees
$T_i(\bm{x})$ of the form for fixed-stepsize cases:
\[
f_k(\bm{x}) = T_0 + \eta \sum _{i \in [k]} T_i(\bm{x})
\]
where $T_0$ is the mean of response variables in the training data,
$\eta$ is the stepsize, and $T_i(\bm{x})$ is the $i$-th regression
tree as a base learner. To fit the model to the training data, GTB
performs the following gradient-descent-like iterations as a boosting
procedure:
\begin{eqnarray*}
  f_0(\bm{x}) \leftarrow \arg \min_c \sum_{i \in [N]} \ell(y_i, c), \\
  \quad\text{and}\quad
  f_k(\bm{x}) \leftarrow f_{k-1}(\bm{x}) + \eta \, T_k(\bm{x}),
\end{eqnarray*}
where $T_k(\bm{x})$ is a regression tree to best approximate the values of negative functional
gradient $-\nabla_{f_{k-1}} L(f_{k-1}) $ at $f_{k-1}$ obtained by
fitting a regression tree to the data $\left\{\left(
\bm{x}_i, r_i
\right)\right\}_{i \in [N]}$, where
\begin{equation*}
  r_i =
  -\left[\frac{\partial\, \ell(y_i,
      f(\bm{x}_i))}{\partial
      f(\bm{x}_i)}\right]_{f(\bm{x})=f_{k-1}(\bm{x}_i)}
\end{equation*}
Our experiments focus on binary classification tasks with $y \in
\{-1,+1\}$, and thus we use the logistic loss $\ell(y, \mu) = \log
(1+\exp(-2 y \mu))$. Note that even for classification, GTB must fit
regression trees instead of classification trees in order to approximate
real-valued functions.

The primary hyperparameters of GTB that we consider are
the following three parameters: a max tree-depth $d$, a stepsize $\eta$,
and the number of trees $k$.

\subsection{Regression Trees}

The internal regression-tree fitting is performed by the recursive
partitioning below:
\begin{enumerate}
\item
  Each node in the regression tree receives a subset $\D' \subseteq \D$ from the parent node.
\item
  If a terminal condition is satisfied,
  the node becomes a leaf decision node with a prediction
  value by the average of the response values in $\D'$.
\item
  Otherwise, the node becomes an internal node that tries
  to find the best partition of $\D'$ to $\D_1$ and
  $\D_0 = \D' \setminus \D_1$ that minimizes the total
  sum of squares of residual error $r$:
  \begin{equation*}
    \myargmin{\D_1, \D_0} \bigl[ \tss(\D_1) + \tss(\D_0) \bigr] . 
  \end{equation*}
  The subsets $\D_1$ and $\D_0$ are further sent to the child nodes, and
  Step (1) is then recursively applied to each subset at the child nodes.
\end{enumerate}
$\tss$ here is the total sum of squares of residual error $r$:
\begin{equation}
  \label{eq:tss} 
  \tss(\D) = \frac{1}{2} \sum_{i \in [N]}
  (r_i - \bar{r})^2, \,\, \bar{r} = \frac{1}{N} \sum_{i \in [N]}
  r_i. 
\end{equation}

\section{Problem Setting and Challenges}

Our goal is learning a nonlinear model $f$ over all possible subgraph
indicators $\I(G \sqsupseteq g)$ for $g \in \S_N$. As we will see in Section
\ref{sec:linearNonlinear}, arbitrary functions of subgraph indicators
have a unique multi-linear polynomial form 
\begin{equation*}
  f(G) = \sum_{S \subseteq \S_N} c_S \prod_{g \in S} \I(G \sqsupseteq g) . 
\end{equation*}
Input graphs are implicitly represented as a bag of subgraph features,
and hence as feature vectors in which the elements are an intractably
large number of binary variables of each subgraph indicator. 
The main challenge is how
to learn the relevant features $g$ from such combinatorially large space $\S_N$
with also jointly learning the classifier $f$ over those
features $\I(G \sqsupseteq g)$ for $g \in \S_N$. Other technical challenges are (1) feature vectors are binary
valued and takes finite discrete values only at the vertices of a very
high-dimensional Boolean hypercube; (2) feature vectors are strongly
correlated due to subgraph isomorphism.

\section{Pseudo-Boolean Functions of Substructural Indicators}
\label{sec:linearNonlinear}

\newcommand{\FL}{\mathcal{F}_L}
\newcommand{\FNL}{\mathcal{F}_{NL}}
\newcommand{\fL}{f_L}
\newcommand{\fNL}{f_{NL}}

We first investigate the difference between linear and nonlinear models.
Any subgraph indicator $\I(G \sqsupseteq g)$ is a 0-1 Boolean
variable, and thus the hypothesis space that we can consider with respect to
these variables is a family of \textit{pseudo-Boolean functions},
regardless of whether they are linear or nonlinear.
A real-valued function $f: \{0,1\}^d \to \mathbb{R}$ on the Boolean hypercube
$\{0,1\}^d$ is called pseudo-Boolean. 

In this section, we explain
the inequivalence of linear and nonlinear models of all possible
subgraph indicators. Theorem~\ref{thm:linear_nonlinear}, contrasting the
difference from closely related problems for itemsets, suggests
an advantage of the proposed nonlinear approach, and an illustrative
example we call Graph-XOR indicating that linear models cannot learn is
presented in Section \ref{sec:artiExperi}.
\begin{theorem}
  \label{thm:linear_nonlinear}
  (1)  The hypothesis space of the nonlinear model of
  all possible sub-itemset indicators is equivalent to that
  of the linear model. (2) The hypothesis space of the
  nonlinear model of all possible connected subgraph
  indicators is strictly larger than that of the linear model.
\end{theorem}
This result is based on the following fundamental property of
pseudo-Boolean functions.
\begin{lemma}
  \label{le:pseudoBooleanFunc}
  \cite{Hammer1963,hammer1968boolean}
  Every pseudo-Boolean function $f : \{0,1\}^d \to \mathbb{R}$ has
  a unique multi-linear polynomial representation:
  \[
  f(x_1, \dots, x_d) = \sum_{S \subseteq [d]} c_S \prod_{j \in S} x_j,
  \quad x_j \in \{0,1\},\,\, c_S \in \mathbb{R}.
  \]
\end{lemma}

\subsection{Sub-Itemset Indicators}

Let $x_j \in \{0,1\}$ be a Boolean variable defined by $x_j = \I(j \in
I)$ for an item $j \in [d]$ in a itemset $I \subseteq [d]$. Then we can
see that linear and nonlinear models of \textit{sub-itemset indicators}
$\I(S \subseteq I), S \subseteq [d]$ are equivalent as a hypothesis
space on itemsets.
For any function $f: 2^{[d]} \to \mathbb{R}$, we have
\begin{align*}
  f(I) &= \sum_{P \subseteq 2^{[d]}} c_P \prod_{S \in P} \I(S \subseteq I)
  = \sum_{U \in 2^{[d]}} c_U \I(U \subseteq I)\\
  &= \sum_{U \subseteq [d]} c_U \prod_{j \in U}\I(j \in I) 
\end{align*}
where $U = \bigcup_{S \in P} S $. Theorem \ref{thm:linear_nonlinear} (1)
follows from this simple fact.
Note that
including the negation terms, as in decision tree learning,
does not change the hypothesis space
because it can be represented as $\I(S \not\subseteq I) = 1 - \I(S \subseteq I)$.

\subsection{Connected-Subgraph Indicators}
\label{sec:ConnectedGraph}

As for subgraph indicators $\I(G \sqsupseteq g)$, the standard setting
implicitly assumes that subgraph feature $g$ is a \textit{connected}
graph. This would be primarily because the complete search for subgraph
patterns, including disconnected graphs, is practically impossible, given
that even a set of all connected graphs $\S_N$ in $\Graph_N$,
is already intractably huge in practice.

The difference between linear model $\fL(G)$ and nonlinear model $\fNL(G)$ is not constantly zero:
\begin{eqnarray*}
  \label{eq:ConnectedGraphLinear}
  \fL(G) &=&  c_0 + \sum_{g \in \S_N} c_g \I(G \sqsupseteq g) \\
  \fNL(G) &=& c_0 + \sum_{g \in \S_N} c_g \I(G \sqsupseteq g) \\
  ~ &~& + \sum_{S \subseteq \S_N, |S| \geqslant 2} c_S \prod_{g \in S} \I(G \sqsupseteq g)
\end{eqnarray*}
from which Theorem \ref{thm:linear_nonlinear} (2) follows.
This also implies that if we consider \textit{connected-subgraph-set indicators}
for the co-occurrence of several connected subgraphs,
then the linear model is equivalent to the nonlinear model as a
hypothesis space, and more importantly it is identical to $\fNL(G)$,
which is the hypothesis space
covered by the proposed algorithm in Section \ref{sec:proposed}.
Note that connected-subgraph-set indicators differ from
the indicators of general subgraphs, including disconnected-subgraph-set indicators, because
any subgraph feature $g \in \S$ can occur multiple times at different
partially overlapped locations in a single graph. The hypothesis space
by general subgraph indicators is beyond the scope of the present paper,
and, in practice, the complete search for such indicators is computationally too challenging.

\section{Proposed Method}
\label{sec:proposed}

In this section, we present a novel efficient method to produce a nonlinear
prediction model based on gradient tree boosting with all possible
subgraph indicators. Existing boosting-based methods \cite{Kudo:2005, Nowozin:2007, Saigo:2009} are based on simple but
efficiently searchable base-learners of decision stumps (equivalent to
subgraph indicators, as demonstrated previously) and construct an efficient
pruning algorithm for this single best subgraph search at each iteration.
In contrast, the proposed approach involves this subgraph
search at finding
an optimal split at each internal node of
regression trees, while keeping the other outer loops the same as in
GTB, as explained in Section \ref{sec:GTB}. More specifically, we need to
efficiently perform the following optimization over all possible
subgraphs in $\S_N$:
\begin{equation}
  \label{eq:internal}
  \myargmin{g \in \S_N} \Bigl[ \tss(\D_1(g)) + \tss(\D_0(g)) \Bigr]
\end{equation}
where $\tss$ is the total sum of squares defined as \eqref{eq:tss},
$\D_1(g) = \{ (G_i, r_i) \in \D \mid G \sqsupseteq g \} $ and
$\D_0(g) = \{ (G_i, r_i) \in \D \mid G \not\sqsupseteq g \} $.

Here, $|\S_N|$ is too intractably huge to solve (\ref{eq:internal}) by
exhaustively testing subgraph $g \in \S_N$ in order, and thus we perform
a branch and bound search over the enumerate tree on $\S_N$ with the
following lower bound for the total sum of squares of expanded subgraphs:
\begin{theorem}
  \label{thm:bound}
  Given $\D_1(g)$ and $\D_0(g)$, for any subgraph $g'
  \sqsupseteq g$,
  \begin{multline}
    \label{eq:bound}
    \tss(\D_1(g')) + \tss(\D_0(g')) \geq \\
    \mymin{(\diamond,k)} \Big[ \tss(\D_1(g) \setminus S_{\diamond, k}) + \tss(\D_0(g) \cup S_{\diamond, k}) \Big]
  \end{multline}
  where $ (\diamond, k) \in \{ \leq, > \} \times \{ 2, \dots, |\D_1(g) - 1| \} $,
  and $S_{\diamond, k} \subset \D_1(g)$, such that $S_{\leq, k}$ is
  a set of $k$ pair $(G_i, r_i)$ selected from $\D_1(g)$ in descending order of residual error $r_i$,
  and $S_{>, k}$ is that in increasing order.
  Note that $\setminus$, $\cup$ are set difference and set union respectively.
\end{theorem}
\begin{proof}
  The result follows from the property
  \eqref{eq:propSubgraph}. See Appendix \ref{sec:boundLinear} for details.
\end{proof}

The entire procedure of the proposed algorithm is illustrated in
Alg.~\ref{alg:gtb} and \ref{alg:rt}. The novel algorithm
for the optimal subgraph search of \eqref{eq:internal} using the bound
\eqref{eq:bound} is described in detail in Alg.~\ref{alg:proposed}.
In order to solve \eqref{eq:internal}, the proposed algorithm uses a depth-first
search on the enumerate tree over $\S_N$. The procedure at each subgraph \g is as follows:
\begin{enumerate}
\item Calculate the total sum of squares \t $\gets \tss(\D_1(g) +
  \D_0(g))$ of subgraph $g$.
\item Update \mt by \t if \mt $>$ \t.
\item Calculate $\bound$ \eqref{eq:bound} and if \mt $<$
  $\bound$, then prune all child nodes of \g.
\end{enumerate}

\SetKwProg{Fn}{Function}{}{}
\SetKwComment{cmt}{$\triangleright$~}{}
\begin{algorithm2e}[t]
  \caption{Gradient Tree Boosting for Graphs }
  \label{alg:gtb}
  \KwIn{Training data $\D = \{ (G_1,y_1), (G_2,y_2), \dots,
    (G_N,y_N) \} $, and stepsize $\eta$}
  \KwOut{Prediction model $f: G \to \mathcal{Y}$}
  \SetKwFunction{GradientTreeBoosting}{GradientTreeBoosting}
  \SetKwFunction{BuildRegressionTree}{BuildRegressionTree}
  \Fn{\GradientTreeBoosting($\D$)}{
    $f \gets 1/N \sum_{i=1}^{N} y_i$ \;
    \For{$k = $1, 2, \dots}{
      \For{$i \gets 1$ \KwTo $N$}{
	$r_i^{(k)} \gets -\left[\frac{\partial\, \ell(y_i,
	    T(G_i))}{\partial
	    T(G_i)}\right]_{T(G)=T_{k-1}(G_i)}$
      }
      $T_k \leftarrow$ \BuildRegressionTree($\{(G_i,r_i^{(k)})\mid i \in [N]\}$)   \cmt*{Alg.~\ref{alg:rt}}
      $f \leftarrow f + \eta \, T_k$ \;
    }
    \KwRet{$f$}
  }
\end{algorithm2e}

\begin{algorithm2e}[t]
  \caption{Regression Tree Learning for Graphs }
  \label{alg:rt}
  \KwIn{Training data $\D = \{ (G_1,r_1), (G_2,r_2), \dots, (G_N,r_N) \} $}
  \KwOut{Regression tree $T$}
  \SetKwFunction{FindBestSplit}{FindBestSplit}
  \SetKwFunction{BuildRegressionTree}{BuildRegressionTree}
  \Fn{\BuildRegressionTree($\D$)}{
    \eIf{\text{the terminal condition is satisfied}}{
      make a leaf node in $T$ with the mean of $r_i$ \;
    }{
      $g \leftarrow$ \FindBestSplit($\D$)    \cmt*{Alg.~\ref{alg:proposed}}
      make an internal node $v$ in $T$ with $g$ \;
      the left child of $v$ $\gets$ \BuildRegressionTree($\D_1(g)$) \;
      the right child of $v$ $\gets$ \BuildRegressionTree($\D_0(g)$) \;
    }
    \KwRet{T}
  }
\end{algorithm2e}

\begin{algorithm2e}[t]
  \caption{Optimal Subgraph Search for Best Split}
  \label{alg:proposed}
  \KwIn{Training data $\D = \{ (G_1,r_1), (G_2,r_2), \dots, (G_N,r_N) \} $}
  \KwOut{Minimizer subgraph $g^*$ of \eqref{eq:internal}}
  \SetKwFunction{FindBestSplit}{FindBestSplit}
  \Fn{\FindBestSplit($\D$)}{
    \Repeat{the enumeration tree search ends}{
      \g $\leftarrow $ the next node of the enumeration tree by DFS \;
      \t $\leftarrow \tss(\D_1(g)) + \tss(\D_0(g))$ \;
      \If{\t $< \mt$}{
	\mt $\leftarrow$ \t \;
	$g^*$ $\leftarrow$ \g  \;
      }
      $\bound \gets$ $\min_{(\diamond,k)} [ \tss(\D_1(g) \setminus S_{\diamond, k}) + \tss(\D_0(g) \cup S_{\diamond, k}) ]$
      \cmt*{Theorem \ref{thm:bound}}
      \If{\mt $< \bound$}{
	prune all children of \g in the enumeration tree \;
      }
    }
    \KwRet{$g^*$} \;
  }
\end{algorithm2e}

In the entire procedure, the most time-consuming part is the subgraph
search (Alg.~\ref{alg:proposed}), which
is repeatedly called during the learning process. Hence, introducing
memorization, whereby we store expensive calls and return the cached
results when the same pattern occurs again, can considerably speed up
the entire process. First, we can store the result of minimization
$\mymin{g \in \S_N} (\tss(\D_1(g)) + \tss(\D_0(g)))$ for each
already checked $g$ and $\D$. Second, the subgraph search can be
entirely skipped until we need to check any subgraph that has not been
checked in any previous iterations.

\section{Numerical Experiments}
\label{sec:experi}

\subsection{The Graph-XOR Problem}
\label{sec:artiExperi}
The linear separability has long been discussed using the
XOR (or parity in general) example, and we present the
same key example for graphs, referred to as \textit{Graph-XOR}, where
linear models cannot learn the target rule even when noiseless examples are provided.

The Graph-XOR dataset includes 1,035 graphs of seven nodes and six edges, where
506 are positives with $y=+1$ and 529 negatives with $y = -1$. As
illustrated in Figure~\ref{fig:artificialGraph}, each graph is generated
by connecting two subgraphs by one node \v{D}. The component subgraphs
are selected from the 18 types shown in Figure~\ref{fig:parts}, where all
three-node path graphs with candidate nodes \{\v{A}, \v{B},
\v{C}\}, and are randomly classified into two groups. Note that
$\v{A}\e\v{B}\e\v{C}$ is isomorphic to
$\v{C}\e\v{B}\e\v{A}$ and this duplicate redundancy due to the graph
isomorphism is removed. The response value $y$ of a graph is $-1$ if two
subgraphs are selected from the same group, otherwise $+1$.

Table~\ref{tbl:resultArti} shows the performance results for the
Graph-XOR data by two-fold cross validations. 
We use the proposed nonlinear method
and two linear methods, namely,
the proposed algorithm with maximum tree-depth ($d$) $=1$, i.e., with decision stumps,
and a state-of-the-art (but linear) method for graphs, gBoost \cite{Saigo:2009}.
The hyperparameter tuning is performed for the ranges
described in Table~\ref{tbl:paramForArti}, and the best parameters are
also listed in Table~\ref{tbl:resultArti}. Figures~\ref{fig:artiByd} and \ref{fig:artiByx} show the accuracy and loss changes for the test
data with regard to the max tree depth ($d$) and  the max subgraph size ($x$).
Here ``subgraph size'' means the number of edges.

The results shown in Table~\ref{tbl:resultArti} clearly demonstrate that
the linear models, including our model with $d=1$, fail, but the nonlinear
methods work well. This is also theoretically supported
by Theorem \ref{thm:linear_nonlinear}.
Figure~\ref{fig:artiByd} also shows that
only the behavior of $d=1$ differ from those of the other depths. Moreover,
note that this problem at least requires subgraph features of size 2 (i.e., two edges),
but searching excessively large subgraphs results in overfitting,
as we see for $x$ $\geqslant 4$ in Figure~\ref{fig:artiByx}.

\begin{figure}[t]
  \begin{minipage}[c]{0.50\hsize}
    \centering
    \begin{minipage}[t]{0.45\hsize}
      \centering
      \includegraphics[width=0.8\textwidth]{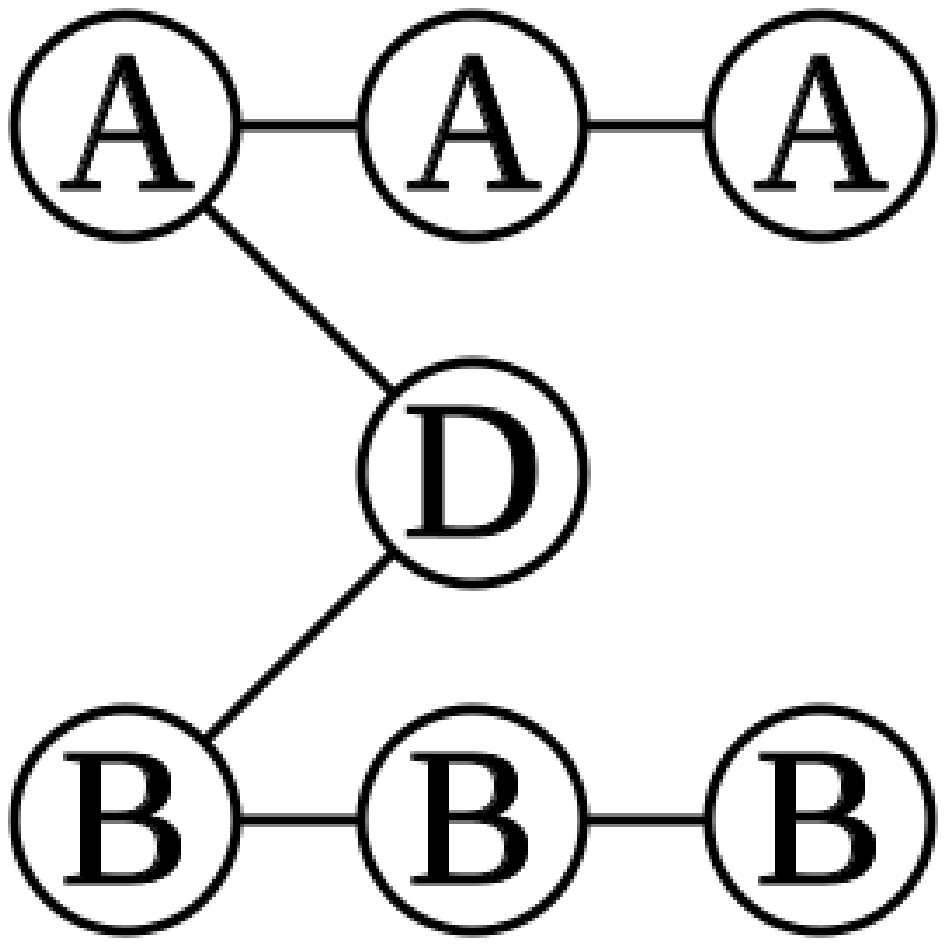} \\
      $y = +1$
    \end{minipage}
    \vspace*{8pt}
    \begin{minipage}[t]{0.45\hsize}
      \centering
      \includegraphics[width=0.8\textwidth]{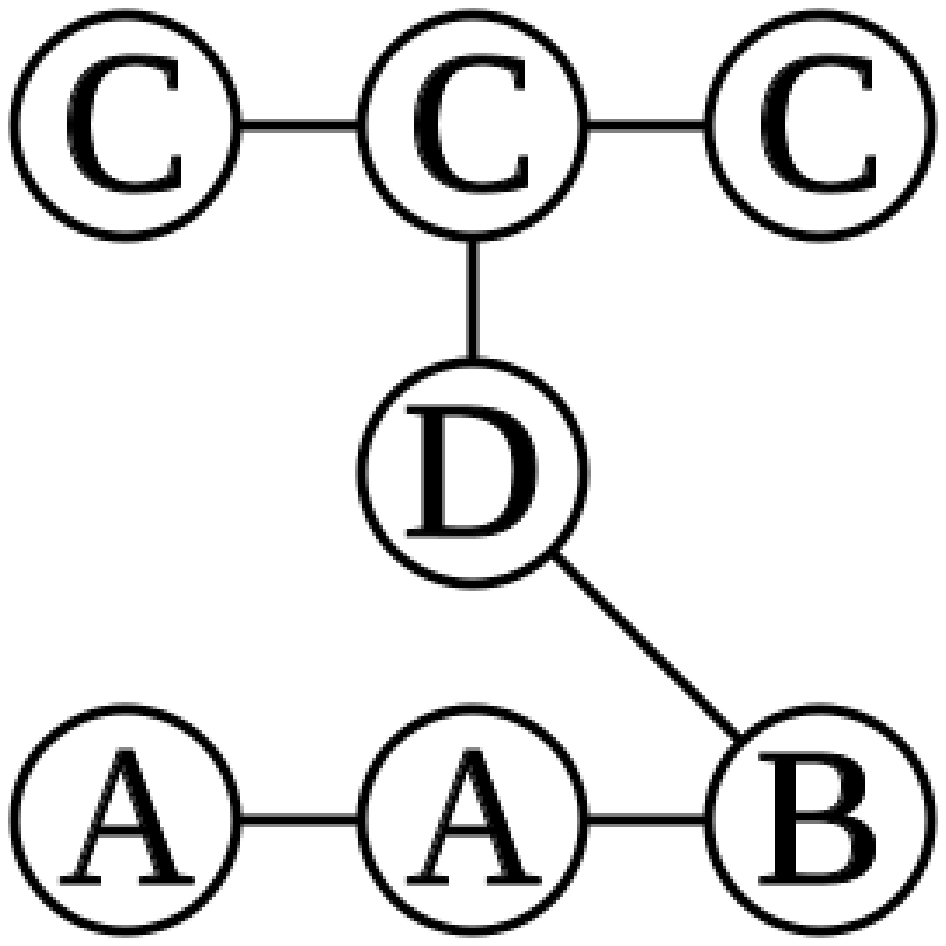} \\
      $y = +1$
    \end{minipage} 
    \begin{minipage}[t]{0.45\hsize}
      \centering
      \includegraphics[width=0.8\textwidth]{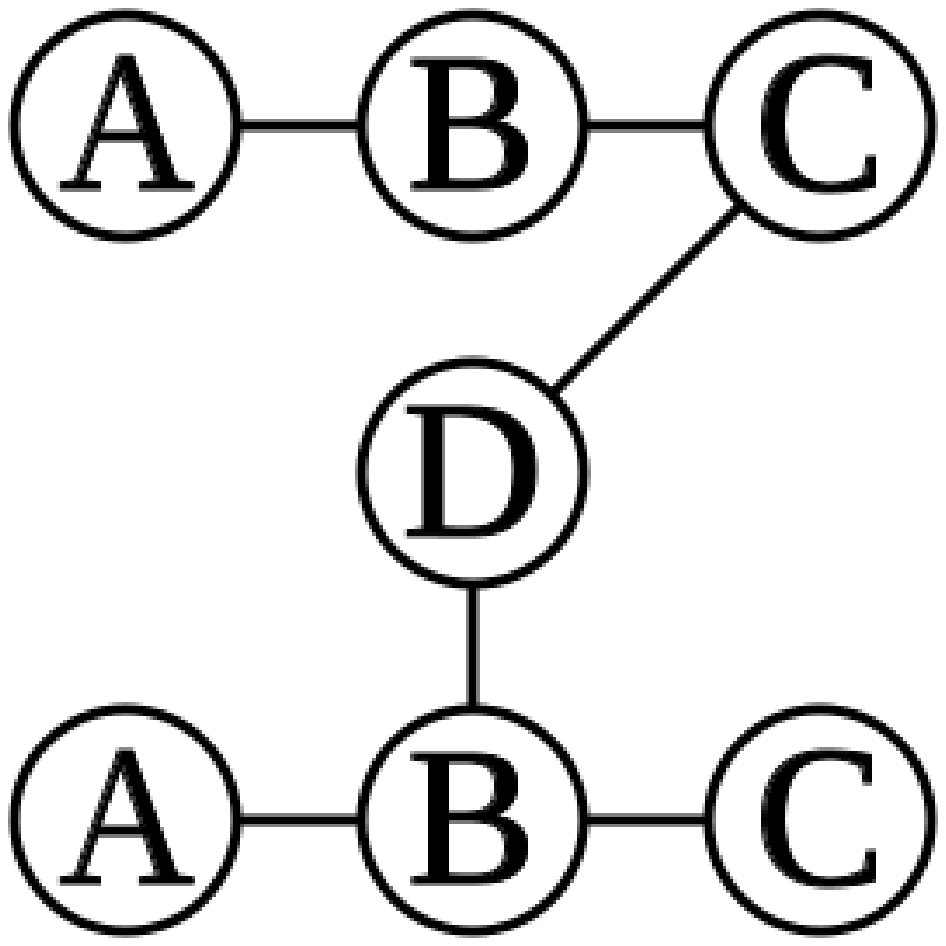} \\
      $y = -1$
    \end{minipage}
    \begin{minipage}[t]{0.45\hsize}
      \centering
      \includegraphics[width=0.8\textwidth]{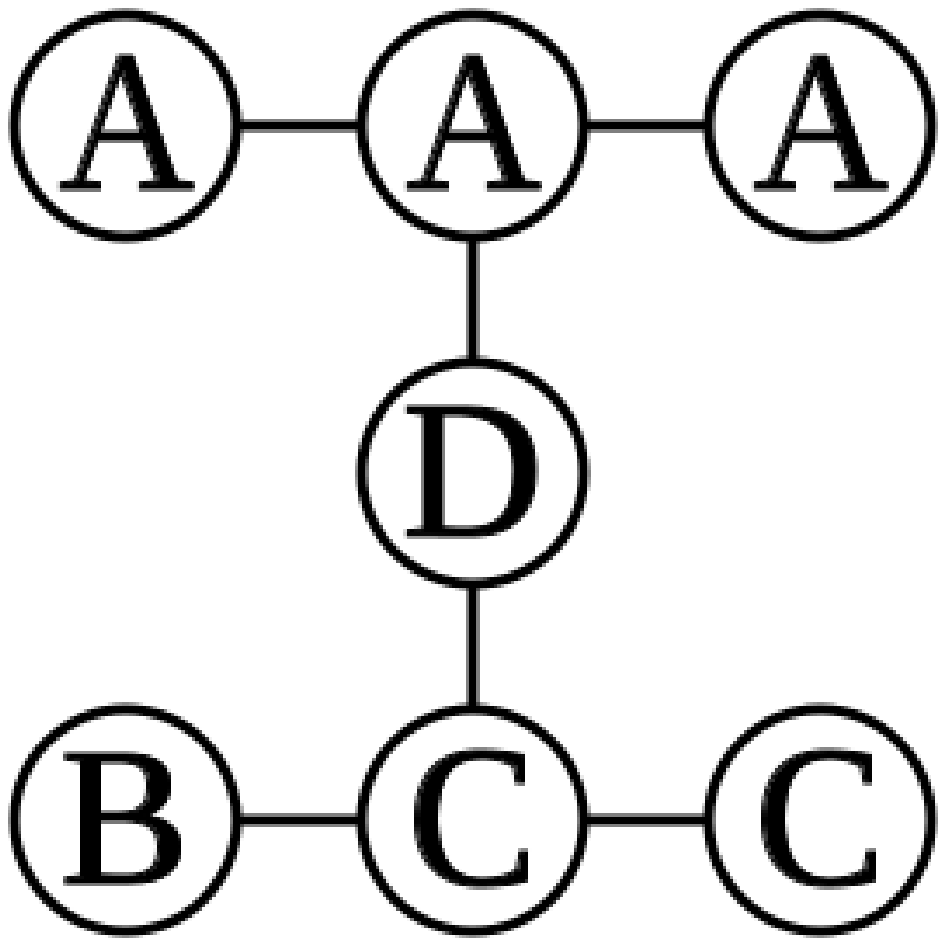} \\
      $y = -1$
    \end{minipage} \\
    \caption{Examples of the Graph-XOR data}
    \label{fig:artificialGraph}
  \end{minipage}
  \begin{minipage}[c]{0.49\hsize}
    \centering
    \begin{tabular}{ll}
      \thickhline
      Group 1 & Group 2 \\
      \hline
      \gxxx{A}{A}{A} & \gxxx{B}{B}{B} \\
      \gxxx{C}{C}{C} & \gxxx{A}{A}{B} \\
      \gxxx{A}{B}{B} & \gxxx{A}{B}{A} \\
      \gxxx{B}{A}{B} & \gxxx{B}{B}{C} \\
      \gxxx{B}{C}{C} & \gxxx{B}{C}{B} \\
      \gxxx{C}{B}{C} & \gxxx{A}{A}{C} \\
      \gxxx{A}{C}{C} & \gxxx{A}{C}{A} \\
      \gxxx{C}{A}{C} & \gxxx{A}{B}{C} \\
      \gxxx{A}{C}{B} & \gxxx{B}{A}{C} \\
      \thickhline
    \end{tabular}
    \caption{Subgraph groups}
    \label{fig:parts}
  \end{minipage}
\end{figure}

\begin{table}[t]
  \centering
  \caption{Hyperparameter settings for Graph-XOR}
  \label{tbl:paramForArti}
  \begin{tabular}{llcl}
    \multicolumn{3}{l}{\textit{Common}} \\
    \thickhline
    \multicolumn{2}{l}{max subgraph size (edges)} & $x$ & 2, 3, 4, 5, $\infty$ \\
    \thickhline
    \multicolumn{3}{l}{} \\
    \multicolumn{3}{l}{\textit{Model specific}} \\
    \thickhline
    Proposed & max tree depth & $d$ & 1, 2, 3, 4, 5 \\
    ~ & stepsize & $\eta$ & 1, 0.7, 0.4, 0.1, 0.01 \\
    ~ & \# trees & $k$ & 1-500 \\
    \hline
    gBoost & regularization & $\nu$ & \ss{0.6, 0.5, 0.4, 0.3, }{0.2, 0.1, 0.01} \\
    \thickhline
  \end{tabular}
\end{table}

\begin{table}[t]
  \centering
  \caption{Prediction accuracy (\%) for the Graph-XOR}
  \label{tbl:resultArti}
  \begin{tabular}{lllll}
    \thickhline
    \multicolumn{1}{l}{\textit{Nonlinear models}} &
    \multicolumn{2}{l}{\textit{Linear models}} \\ \hline
    Proposed & Proposed ($d1$) & gBoost \\
    \hline
    \fst{100.0} & 64.3  & 70.0 \\
    $x2~d2~\eta0.7~k221$  & $x6~d1~\eta0.7~k26$    & $x6~\nu0.01$ \\
    \thickhline
  \end{tabular}
\end{table}
 
\begin{figure}[h!]
    \centering
    \includegraphics[width=0.45\textwidth]{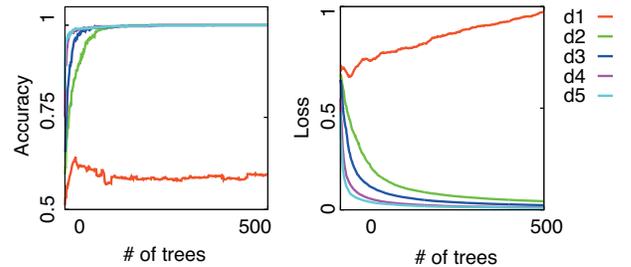}
    \caption{Test accuracy and loss with tree depth $d$}
    \label{fig:artiByd}
\end{figure}

\begin{figure}[h!]
    \centering
    \includegraphics[width=0.45\textwidth]{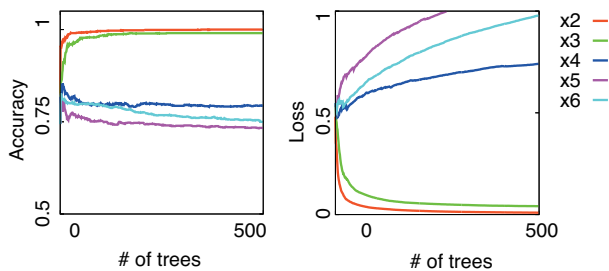}
    \caption{Test accuracy and loss with subgraph size $x$}
    \label{fig:artiByx}
\end{figure}

\subsection{QSAR with Molecular Graphs}

We also evaluate the performance based on the most typical benchmark for
graph classification on real datasets: the quantitative
structure-activity relationship (QSAR) results with molecular graphs. We
select four binary-classification datasets (CPDB, Mutag, NCI1, NCI47) in
Table~\ref{tbl:dataset}: two data (CPDB, Mutag) for
mutagenicity tests and two data (NCI1, NCI47) for tumor growth
inhibition tests from PubChem
BioAssay\footnote{\url{https://pubchem.ncbi.nlm.nih.gov/bioassay/}$\langle\text{AID}\rangle$
  (AID numbers are 1 and 47, respectively)}. NCI1 and NCI47 are balanced by
randomly sampling negatives of the same size as the positives in order to avoid
imbalance difficulty in evaluation. All chemical structures are encoded
as molecular graphs using RDKit\footnote{\url{http://www.rdkit.org/}}, and
some structures in the raw data are removed by chemical sanitization\footnote{
Due to this pre-processing, the number of datasets differs
  from that in the simple molecular graphs in the literature, where the nodes are
  labeled by atom type, and the edges are labeled by bond type.}.
We simply apply a node labeling by the RDKit default atom invariants
(edges not labeled), i.e., atom type, \# of non-H neighbors, \# of Hs,
charge, isotope, and inRing properties. These default atom invariants use
connectivity information similar to that used for the well-known ECFP
family of fingerprints\cite{Rogers:2010}. See \cite{Kearnes2016} for more
elaborate encodings.

Tables~\ref{tbl:acc} and ~\ref{tbl:bestParam} show the performance
results obtained by 10-fold cross validations using the same three methods used
for the Graph-XOR cases with different hyperparameter settings in
Table~\ref{tbl:param}. We can observe that nonlinear methods
often outperform the linear methods. At the same time, we can also
observe, in some cases, that the linear methods work fairly well for the
real datasets. The real datasets would not have explicit classification
rules compared to
noiseless problems such as the Graph-XOR cases. Thus, it is necessary to
tolerate some noises and ambiguity. Although they may seem limited,
linear hypothesis classes are known to be very powerful in such cases,
because they are quite stable estimators and the input features can
themselves include nonlinear features of data as implied in Theorem~\ref{thm:linear_nonlinear}.

We also provide the normalized feature importance scores from GTB and
the search space size in Figure~\ref{fig:FInSS} for the CPDB dataset.
In Figure~\ref{fig:FInSS}, {\it searched} corresponds to the searched subgraphs,
and {\it selected} to the subgraph selected as internal nodes. 
This would also implies that (i) the proposed approach can provide information on selected relevant subgraph
features and (ii) searches and uses only a portion of the entire search space.

\begin{table}[t]
  \centering
  \caption{Dataset summary}
  \label{tbl:dataset}
  \begin{tabular}{lccccc}
    \thickhline
    Dataset                   & \ss{Graph-}{XOR}       & CPDB           & Mutag        & NCI1             & NCI47            \\  \hline
    \# data                & 1035            & 600            & 187          & 4252             & 4202             \\
    \# nodes    & 7               & 13.7           & 17.9         & 26.3             & 26.3             \\  
    \# edges       & 6               & 14.2           & 19.7         & 28.4             & 28.4             \\  \thickhline
  \end{tabular}
  \leftline{\hspace*{22pt} \# of nodes and edges are average.}
\end{table}

\begin{table}[t]
  \centering
  \caption{Hyperparameter settings for the QSAR}
  \label{tbl:param}
  \begin{tabular}{llcl}
    \multicolumn{3}{l}{\textit{Common}} \\
    \thickhline
    \multicolumn{2}{l}{max subgraph size (\# edges)} & $x$ & 4, 6, 8 \\
    \thickhline
    \multicolumn{3}{l}{} \\
    \multicolumn{3}{l}{\textit{Model specific}} \\
    \thickhline
    Proposed & max tree depth & $d$ & 1, 3, 5 \\
    ~ & stepsize & $\eta$ & 1.0, 0.7, 0.4, 0.1 \\
    ~ & \# trees & $k$ & 1-500 \\
    \hline
    gBoost & regularization & $\nu$ & \ss{0.6, 0.5, 0.4, 0.3,}{ 0.2, 0.1, 0.01} \\
    \thickhline
  \end{tabular}
\end{table}

\begin{table*}[t]
  \centering
  \caption{Prediction accuracy (\%) for the QSAR}
  \label{tbl:acc}
  \begin{tabular}{lcccccccc}
    \thickhline
    ~                          &        CPDB                 &        ~                        &        Mutag              &        ~                   &        NCI1                &  ~                        &        NCI47                    &        ~                       \\
    ~                          &        ACC                  &        AUC                      &        ACC                &        AUC                 &        ACC                 &  AUC                      &        ACC                      &        AUC                     \\ \hline
    \multicolumn{9}{l}{\textit{Nonlinear models}} \\
    \ss{Proposed}{~}                   &   \fss{79.3}{($\pm$4.8)}    &   \ss{84.5}{($\pm$3.6)}        &    \ss{87.8}{($\pm$6.6)}  &    \ss{91.6}{($\pm$6.3)}   &   \fss{84.7}{($\pm$1.7)}   &  \fss{90.8}{($\pm$1.3)}   &   \fss{84.5}{($\pm$1.7)}        &   \fss{90.3}{($\pm$1.1)}       \\ \hline
    \multicolumn{9}{l}{\textit{Linear models}} \\
    \ss{Proposed (d1)}{~}              &    \fss{79.3}{($\pm$4.4)}    &    \ss{83.9}{($\pm$3.3)}        &   \ss{87.8}{($\pm$6.6)}  &   \ss{91.6}{($\pm$6.3)}   &    \ss{83.1}{($\pm$1.6)}   &   \ss{89.8}{($\pm$1.3)}   &        \ss{82.8}{($\pm$1.4)}    &    \ss{88.9}{($\pm$1.1)}       \\
    \ss{gBoost}{~}                     &    \ss{77.1}{($\pm$2.7)}    &    \ss{73.6}{($\pm$4.9)}        &   \fss{91.4}{($\pm$5.8)}  &   \fss{93.9}{($\pm$5.0)}   &    \ss{82.7}{($\pm$2.2)}   &   \ss{83.9}{($\pm$2.2)}   &        \ss{81.3}{($\pm$1.4)}    &    \ss{81.8}{($\pm$2.6)}       \\
    \hline
    \multicolumn{9}{l}{\textit{Reported values in literature}} \\
    L1-LogReg \cite{takigawa:2017}     &        78.3                 &        -                     &        -               &        -                &        -                   &  -                        &        -                        &        -                       \\

    MGK \cite{Saigo:2009}              &        76.5                 &        75.6                     &        80.8               &        90.1                &        -                   &  -                        &        -                        &        -                       \\
    freqSVM \cite{Saigo:2009}	       &        77.8                 &        84.5                     &        80.8               &        90.6                &        -                   &  -                        &        -                        &        -                       \\
    gBoost \cite{Saigo:2009}	       &        78.8                 &   \fst{85.4}                    &        85.2               &        92.6                &        -                   &  -                        &        -                        &        -                       \\ 
    WL shortest path \cite{Shrvashidze:2011}        &        -                    &        -                        &        83.7               &        -                   &        84.5                &  -                        &        -                        &        -                       \\
    Random walk \cite{Shrvashidze:2011}             &        -                    &        -                        &        80.7               &        -                   &        64.3                &  -                        &        -                        &        -                       \\
    Shortest path \cite{Shrvashidze:2011}           &        -                    &        -                        &        87.2               &        -                   &        73.4                &  -                        &        -                        &        -                       \\ \thickhline
  \end{tabular}
\end{table*}

\begin{table*}[t]
  \centering
  \caption{Best hyperparameters}
  \label{tbl:bestParam}
  \begin{tabular}{lllll}
    \thickhline
    ~                    & CPDB               & Mutag            & NCI1               & NCI47              \\  \hline
    Proposed             & $x4~d5~\eta0.1~k120$ & $x4~d1~\eta1~k22 $ & $x4~d5~\eta0.1~k452$ & $x4~d3~\eta0.4~k308$ \\  
    Proposed(d1)         & $x8~d1~\eta0.4~k128$ & $x4~d1~\eta1~k22 $ & $x4~d1~\eta0.4~k499$ & $x4~d1~\eta0.4~k499$ \\ 
    gBoost               & $x8~\nu0.5         $ & $x7~\nu0.1       $ & $x8~\nu0.3         $ & $x8~\nu0.4         $ \\  \thickhline
  \end{tabular}
\end{table*}

\begin{figure}[t]
  \centering
  \includegraphics[width=0.5\textwidth]{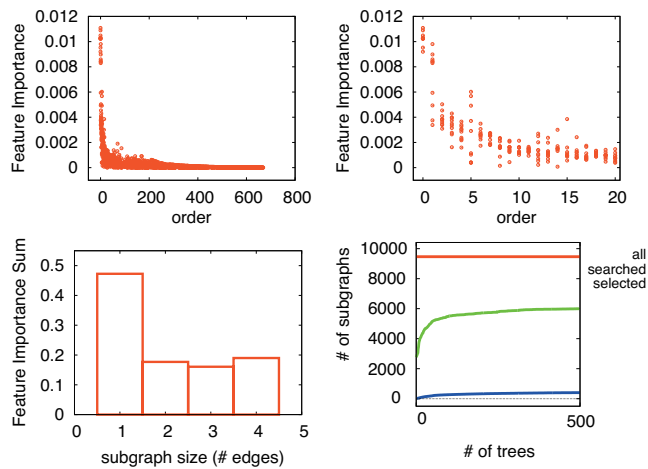}
  \caption{Feature importance and search space for CPDB}
  \label{fig:FInSS}
\end{figure}


\subsection{Scalability comparison to Na\"ive approach}\label{sec:scalability}

As previously mentioned in Section~\ref{sec:previousModel}, there exists a simple na\"ive two-step approach to obtain nonlinear models of subgraph indicators. Figure~\ref{fig:scalability} shows the scalability of this ``enumerate \& learn'' approach by first enumerating all small-size subgraphs and applying general supervised learning to their indicators. The values in the figure are the average values to process each fold in 10-fold cross validation on a single PC with Pentium G4560 3.50GHz and 8GB memory. We enumerate all subgraphs with limited subgraph size\footnote{Small-size subgraphs are known to be more appropriate for this supervised-learning purpose than frequent subgraphs \cite{Wale:2008}.}, and feed their indicator features to GradientBoostingClassifier with 100 trees (depth $\leqslant$ 5) of scikit-learn\footnote{\url{http://scikit-learn.org}}. The proposed method is also tested with the same setting (100 trees, d5). Since this case both use GTB and thus the performance is the same in principle up to implementation details (empirically both 0.75-0.77 for this setting), we focus on scalability comparisons using the fixed hyperparameters. Because the number of subgraph patterns to be enumerated increases exponentially, off-the-shelf packages such as scikit-learn cannot handle them at some point even when pattern enumeration can be done. In Figure~\ref{fig:scalability}, we can observe pattern enumeration can be done for max subgraph size = 1 to 12 (green line, left), but the 2nd scikit-learn step fails for max subgraph size $\geqslant 10$ (green line, right). In this CPDB examples, the numbers of subgraphs, i.e., the dimensions of feature vectors, were 66336.1, 145903.7, 275422.3, 512904.1, 874540.0 for max subgraph size = 8, 9, 10, 11, 12, respectively, and scikit-learn was only feasible for max subgraph size up to 9. Note that we also need to solve a large number of subgraph isomorphism known to be NP-complete.

\begin{figure}[t]
 \centering
  \includegraphics[width=0.49\textwidth]{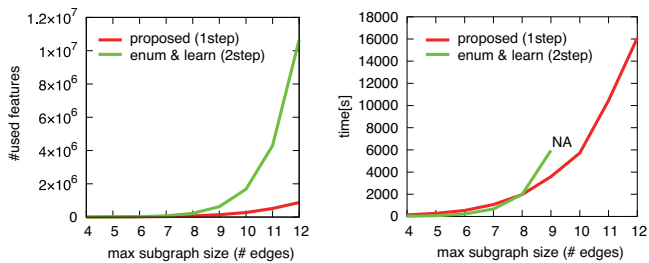}
   \caption{Scalability comparison to na\"ive approach}
   \label{fig:scalability}
\end{figure}

\section{Conclusions}

In summary, we investigated nonlinear models with all possible
subgraph indicators and provide a novel efficient algorithm to learn
from the nonlinear hypothesis space. We demonstrated that this hypothesis space is
identical to the (pseudo-Boolean) functions of these subgraph indicators,
which are, in general, strictly larger than those of the linear models.
This is also empirically confirmed through our Graph-XOR example.
Although most existing studies focus only on real datasets, this would also promote interest in whether graph-theoretic
classification problems can be approximated in a supervised learning manner. 
At the same time, the experimental results of the present study
also strongly suggest that we need a nonlinear hypothesis space for the
QSAR problems based on some real datasets, which would also support a
standard cheminformatics approach of applying nonlinear models, such as
random forests and neural networks, to 0-1 feature vectors, referred to as
\textit{molecular fingerprints}, by the existence of substructural features.

Since research on classification and regression trees originates
from the problem of \textit{automatic interaction detection}
\cite{MorganSonquist:1963, Kass:1975, breiman84}, our approach can provide insights on the question
whether such higher-order interactions between input features exist. 
In this sense, our methods and findings would also be informative to consider 
a recent hot topic of detecting such interactions in combinatorial data \cite{
  terada2013statistical, sugiyama2015significant, nakagawa2015safe}.

\bibliographystyle{ACM-Reference-Format}
\bibliography{all}

\appendix

\section{Proof of Theorem \ref{thm:bound}}
\label{sec:boundLinear}

\renewcommand{\a}{\bar{a}}
\renewcommand{\b}{\bar{b}}
\newcommand{\mysum}[1]{\sum_{i \in[#1]}}

\begin{proof}
  Given $\D_1(g)$ and $\D_0(g)$,
  \small
  \begin{align}
    \bound &= \mymin{g'} \bigl[ \tss(\D_1(g')) + \tss(\D_0(g'))  \bigr] \notag \\
    \label{eq:boundSubset}
    &= \mymin{S \subset \D_1(g)} \bigl[ \tss(\D_1(g) \setminus S) + \tss(\D_0(g) \cup S)  \bigr] \\
    \label{eq:linearBound}
    &= \mymin{(\diamond,k)} \Big[ \tss(\D_1(g) \setminus S_{\diamond,k}) + \tss(\D_0(g) \cup S_{\diamond,k}) \Big]
  \end{align}\normalsize
  where $ (\diamond, k) \in \{ \leq, > \} \times \{ 2, \dots, |\D_1(g) - 1| \} $.
  From the anti-monotone property \eqref{eq:propSubgraph}, we have
  $\D_1(g') \subseteq \D_1(g)$ for $g' \sqsupseteq g$
  for the training set $\D$ from which the equation \eqref{eq:boundSubset}
  directly follows. Thus, we show \eqref{eq:linearBound} in
  detail. For simplicity,
  let $A =,\{ a_1, \dots, a_n \mid a_i \in \mathbb{R} \}$ denote $\D_1(g)$,
  and $B = \{ b_1, \dots, b_m \mid b_i \in \mathbb{R} \}$ denote $\D_0(g)$.
  Then, the goal of \eqref{eq:boundSubset} is to minimize the total sum
  of squares $\tss(A \setminus S) + \tss(B \cup S)$ by tweaking
  $S = \{ s_1, \dots, s_k \} \subset A$.
  Let $\a$, $\a_{-S}$, $\b$, and $\b_{+S}$ be the means of $A$, $A \setminus S$,
  $B$, and $B \cup S$, respectively.
  The key fact is that $\tss(A \setminus S) + \tss(B \cup S)$ can
  be regarded as a quadratic equation of $\sum_{i=1}^k s_i$ when the size
  of $S$ is fixed to $k$. More precisely,
  \begingroup
  \allowdisplaybreaks
      {\small
	\begin{align*}
	  &\tss(A \setminus S) + \tss(B \cup S) \\
	  &= \mysum{n} (a_i - \a_{-S})^2 - \mysum{k} (s_i - \a_{-S})^2 + \mysum{m} (b_i - \a_{+S})^2 + \mysum{k} (s_i - \a_{+S})^2 \\
	  &= - \mysum{k} (s_i - \a)^2 - \frac{\big( \mysum{k} (s_i - \a) \big)^2}{n-k} + \mysum{n} (a_i - \a)^2 \\
	  & \qquad + \mysum{k} (s_i - \b)^2 - \frac{\big( \mysum{k} (s_i - \b) \big)^2}{m+k} + \mysum{m} (b_i - \b)^2 \\
	  & = - \left( \frac{1}{n-k} + \frac{1}{m+k} \right) \Big( \mysum{k} s_i \Big)^2 + \left( 2\a \frac{n}{n-k} - 2\b \frac{m}{m+k} \right) \mysum{k} s_i \\
	  & \qquad- \frac{nk}{n-k} \a^2 + \frac{mk}{m+k} \b^2 + \mysum{n} (a_i - \a)^2 + \mysum{m} (b_i - \b)^2
	\end{align*}
      }
      \endgroup
      Therefore, $\tss(A \setminus S) + \tss(B \cup S)$ is minimized
      when $\sum_{i=1}^k s_i$ is maximized or minimized. In other words,
      \eqref{eq:boundSubset} becomes minimum when the mean of $S \subset
      \D_1(g)$ is maximized or minimized.
\end{proof}

\end{document}